\documentclass[11pt]{article}

\usepackage[utf8]{inputenc}
\usepackage[T1]{fontenc}
\usepackage[margin=1in]{geometry}
\usepackage{microtype}
\usepackage{graphicx}
\usepackage{subcaption}
\usepackage{booktabs}
\usepackage{xcolor}
\usepackage{amsmath}
\usepackage{amssymb}
\usepackage{mathtools}
\usepackage{amsthm}
\usepackage[round,authoryear]{natbib}
\usepackage{hyperref}
\hypersetup{%
  colorlinks=true, linkcolor=blue, citecolor=blue, urlcolor=blue,%
  pdftitle={Improved Distribution Estimation in l-infinity},%
  pdfauthor={Doron Cohen, Aryeh Kontorovich, Yonatan Livshitz}%
}
\usepackage[capitalize,noabbrev]{cleveref}

\makeatletter
\@ifundefined{c@algorithm}{}{}
\makeatother

\theoremstyle{plain}
\newtheorem{theorem}{Theorem}[section]
\newtheorem{proposition}[theorem]{Proposition}
\newtheorem{lemma}[theorem]{Lemma}

\theoremstyle{definition}

\theoremstyle{remark}
\newtheorem{remark}[theorem]{Remark}

\newcommand{\abs}[1]{\left| #1 \right|}
\newcommand{\paren}[1]{\left( #1 \right)}

\newcommand{\N}{\mathbb{N}}

\newcommand{\supr}[1]{^{(#1)}}

\newcommand{\E}{\mathbb{E}}

\newcommand{\beq}{\begin{eqnarray*}}
\newcommand{\eeq}{\end{eqnarray*}}
\newcommand{\beqn}{\begin{eqnarray}}
\newcommand{\eeqn}{\end{eqnarray}}

\newcommand{\eqdef}{:=}

\newcommand{\mathe}{\mathrm{e}}

\newcommand{\kl}[2]{D_{\mathrm{KL}}\left(#1 || #2\right)}
\newcommand{\maavar}[2]{\overset{\textup{(#1)}}{#2}}

\newcommand{\secref}[1]{Section~\ref{#1}}

\newcommand{\lemref}[1]{Lemma~\ref{#1}}
\newcommand{\thmref}[1]{Theorem~\ref{#1}}

\def\epsilon{\varepsilon}

\title{Improved Distribution Estimation in \texorpdfstring{$\ell_{\infty}$}{l-infinity}}

\author{%
Doron Cohen \and Aryeh Kontorovich \and Yonatan Livshitz\\[0.5ex]
Department of Computer Science, Ben-Gurion University of the Negev, Beer-Sheva, Israel\\
\texttt{doronv@post.bgu.ac.il} \quad \texttt{karyeh@bgu.ac.il} \quad \texttt{livshitz@post.bgu.ac.il}
}

\date{}

\begin{document}
\maketitle




    \begin{abstract}
        We present improved bounds for estimating discrete probability distributions
        under the $\ell_{\infty}$ norm. These include minimax bounds in
        expectation and high-probability tail bounds. We resolve some of the open
        questions posed in Kontorovich and Painsky (JMLR, 2025) --- including 
        a fully empirical version of the tightest risk bound they presented and identifying
        the form of the worst-case
        extremal distribution. Encouraging empirical results are reported as well.
    \end{abstract}



    \section{Introduction}
    \label{sec:introduction}

    Estimating an unknown discrete distribution from i.i.d.\ samples is a
    classical problem, with a long history of sharp minimax theory under $\ell_{1}$
    and $\ell_{2}$ losses. In a range of modern applications---including uniform
    calibration, anomaly detection, and goodness-of-fit pipelines that trigger
    on the largest coordinate-wise discrepancy---the relevant performance metric
    is instead the sup-norm loss
    \[
        \|\hat p - p\|_{\infty}\;=\; \sup_{i\ge 1}|\hat p_{i}- p_{i}| ,
    \]
    where $p$ is a probability distribution over a (possibly infinite) countable
    alphabet and $\hat p$ is an estimator based on $n$ samples. 

    The $\ell_{\infty}$ loss is, in a sense, simple: an $O(1/\sqrt{n})$ risk
    decay holds uniformly over all discrete distributions. The technical challenge
    is therefore to obtain delicate, distribution-dependent \emph{fast rates} and
    fully empirical counterparts whose sharpness reflects that the effective difficulty
    is governed by the unknown tail profile of $p$, rather than merely by its
    support size.

    Motivated by various problems in applied statistics, a recent work of \citet{KonPai24}
    initiated the nonasymptotic study of $\ell_{\infty}$ distribution estimation
    in the countable-alphabet regime. In particular, they introduced two
    distribution-dependent complexity proxies that capture the interplay between
    variance and tail decay: \beq v^*(p) &\coloneqq& \sup_{i\ge 1} p_i(1-p_i), \\
    V^*(p) &\coloneqq& \sup_{i\ge 1} p^\downarrow_i(1-p^\downarrow_i)\,\log(i+1)
    \eeq (where $p^{\downarrow}$ is the nonincreasing permutation) and established
    upper bounds (in expectation and high probability) for the MLE $\hat p$ in terms
    of $v^{*}(p)$ and $V^{*}(p)$. They also raised several open questions, including
    (i) whether one can obtain \emph{fully empirical} (data-dependent)
    confidence bounds without unknown functionals of $p$, and (ii) what is the true
    \emph{least-favorable} distribution governing the worst-case risk.

    In this paper we sharpen and extend the theory in both directions. Our results
    give improved minimax bounds (in expectation and deviation form), identify a
    simple extremal distribution that is worst-case up to universal constants
    and can be chosen independently of $n$, and provide a tight \emph{fully
    empirical} high-probability guarantee that depends only on the observed
    sample through natural plug-in statistics. We also report empirical evidence
    demonstrating that these bounds are informative in finite samples.

    \paragraph{Contributions (informal).}
    \begin{itemize}
        \item \textbf{Minimax and least-favorable distribution.} We show that
            the worst-case expected $\ell_{\infty}$ error scales as
            $\Theta(n^{-1/2})$, and that a two-point distribution $p^{\star}=(\tfrac
            12,\tfrac12,0,\dots)$ is least-favorable up to a universal constant factor,
            uniformly for all sufficiently large $n$.

        \item \textbf{Fully empirical confidence bound.} We prove a high-probability
            inequality for $\|\hat p-p\|_{\infty}$ that is \emph{fully empirical},
            depending on $p$ only through plug-in quantities $\hat v^{*}$ and an
            empirical analogue $\hat V^{*}$ of $V^{*}(p)$.

        \item \textbf{Refined distribution-dependent regime picture.} Using
            sharp local Glivenko--Cantelli bounds, we relate the estimation
            error to two explicit complexity terms that exhibit different tail
            behaviors and quantify when each one dominates.
    \end{itemize}

    \section{Main Results}
    \label{sec:main_results}

    \paragraph{Setting.}
    Let $p=(p_{i})_{i\ge 1}$ be a distribution on $\N$ and let $X^{n}=(X_{1},\dots
    ,X_{n})$ be i.i.d.\ from $p$. Let $c_{i}(X^{n})\coloneqq \sum_{t=1}^{n}\mathbf{1}
    \{X_{t}=i\}$ and denote the empirical pmf (MLE) by $\hat p_{i}\coloneqq c_{i}
    (X^{n})/n$. We study the sup-norm error $\|\hat p-p\|_{\infty}$ and its
    expectation
    \[
        \Delta_{n}(p)\;\coloneqq\; \E_{p}\!\left[\|\hat p-p\|_{\infty}\right].
    \]
    Write $p^{\downarrow}$ for $p$ sorted in nonincreasing order, and define
    \begin{align*}
        v_{i}(p) & \coloneqq p_{i}(1-p_{i}),                                \\
        v^{*}(p) & \coloneqq \sup_{i\ge 1}v_{i}(p),                         \\
        V^{*}(p) & \coloneqq \sup_{i\ge 1}v_{i}(p^{\downarrow})\,\log(i+1).
    \end{align*}
    We also define empirical analogues
    \begin{align*}
        \hat v^{*} & \coloneqq \sup_{i\ge 1}\hat p_{i}(1-\hat p_{i}),                \\
        \hat V^{*} & \coloneqq \sup_{i\ge 1}\hat p_{[i]}(1-\hat p_{[i]})\,\log(i+1),
    \end{align*}
    where $\hat p_{[i]}$ are the order statistics of $(\hat p_{i})_{i\ge 1}$.

    We write $f\lesssim g$ if $f\le c g$ where$c$ is an absolute constant.

    \subsection{Minimax behavior and a fixed least-favorable distribution}

    Our first result resolves the extremal structure of the worst-case expected $\ell
    _{\infty}$ risk, answering an open question of \citet{KonPai24}.

    \begin{theorem}[A fixed least-favorable distribution; minimax rate]
        \label{thm:least-favorable} There exist universal constants $c>0$, $C>0$,
        and $n_{0}\in\N$ such that the following holds for all $n\ge n_{0}$. Let
        $p^{\star}\;\coloneqq\; (\tfrac12,\tfrac12,0,0,\dots)$. Then
        \[
            \Delta_{n}(p^{\star})\;\ge\; c\,\Delta_{n}(p)\qquad\text{for all
            distributions }p\text{ on }\N,
        \]
        and consequently
        \[
            c\,n^{-1/2}\;\le\;\sup_{p}\Delta_{n}(p)\;\le\; C\,n^{-1/2}.
        \]
    \end{theorem}

    \paragraph{Remark.}
    Theorem~\ref{thm:least-favorable} shows that, up to universal constants, the
    hardest instance for $\ell_{\infty}$ estimation is already present on a two-point
    alphabet, and does \emph{not} require an $n$-dependent effective support size.
    In particular, the global minimax rate in expectation is $\Theta(n^{-1/2})$,
    with a least-favorable $p^{\star}$ independent of $n$.

    \subsection{A fully empirical high-probability bound}

    Next we give a deviation inequality for $\|\hat p-p\|_{\infty}$ that is \emph{fully
    empirical}, replacing the unknown functionals $V^{*}(p),v^{*}(p)$ by observable
    plug-in quantities.

    \begin{theorem}[Fully empirical $\ell_{\infty}$ bound]
        \label{thm:empirical-Vstar} There exist universal constants $c_{1},c_{2}>
        0$ such that for all $\delta\in(0,1)$, with probability at least
        $1-2\delta-\frac{162}{n}$,
        \[
            \|\hat p-p\|_{\infty}
            \le c_{1}\sqrt{\frac{\hat V^{*}}{n}+\frac{\hat v^{*}}{n}\log\frac{1}{\delta}}
            + c_{2}\left(\frac{\log(n/\delta)}{n}+\frac{\log n}{n}\right).
        \]
    \end{theorem}

    \begin{remark}
        The probability lower bound can be written in a standard $1-2\bar\delta$
        form by setting $\bar\delta \coloneqq \delta + 81/n$, since
        $1-2\delta-162/n = 1-2\bar\delta$. In typical usage $n$ is large and one
        may take $\delta \gtrsim 1/n$ so that the extra term is absorbed.
    \end{remark}

    \paragraph{Remark.}
    The bound in Theorem~\ref{thm:empirical-Vstar} adapts automatically to the
    unknown tail profile of $p$ through $\hat V^{*}$, while retaining the correct
    $\sqrt{\log(1/\delta)/n}$ behavior in the worst case (e.g.\ under $p^{\star}$
    from Theorem~\ref{thm:least-favorable}). Unlike bounds expressed in terms of
    $V^{*}(p)$, it can be evaluated directly from data.

    \subsection{Refined distribution-dependent bounds via local Glivenko--Cantelli}

    To connect with sharp distribution-dependent behavior beyond the global minimax
    regime, we also consider the functionals
    \begin{align*}
        S(p) & \;\coloneqq\;\sup_{j\ge 1}p_{j}^{\downarrow}\,\log(j+1),                                                 \\
        M(p) & \;\coloneqq\;\sup_{j\ge 1}\frac{\log(j+1)}{\log\!\Bigl(2+\frac{\log(j+1)}{n\,p_{j}^{\downarrow}}\Bigr)}.
    \end{align*}
    These quantities arise naturally in sharp local Glivenko--Cantelli theory \citep[cf.][]{CohenK23,blanchard2023tight,cohen25bEME}
    and yield a complementary ``regime picture'' for $\|\hat p-p\|_{\infty}$.

    \begin{theorem}[High-probability bound via LGC functionals]
        \label{thm:lgc-bound} There exists a universal constant $C > 0$ such that
        for all $\delta \in (0, e^{-2}]$, with probability at least $1-2\delta$,
        \[
            \|\hat p - p\|_{\infty}\;\le\;C\,\left(\sqrt{\frac{S(p)}{n}}\vee \frac{M(p)}{n}
            \right)\,\log\frac{1}{\delta}.
        \]
    \end{theorem}

    \begin{lemma}[Size of $M(p)$ and examples]
        \label{lem:M-behavior} There are universal constants $C_{1},c_{1}>0$ such
        that, for every distribution $p$ on $\N$ and every $n\in \N$,
        \[
            c_{1}\;<\;M(p)\;<\; C_{1}\log n.
        \]
        Moreover, there are universal constants $C_{2},c_{2}>0$ and there exist distributions
        $p_{1},p_{2}$ on $\N$ such that for all $n\in\N$,
        \[
            M(p_{1})\;>\; C_{2}\log n, \qquad M(p_{2})\;<\;c_{2}.
        \]
    \end{lemma}

    \begin{lemma}[Size of $S(p)$]
        \label{lem:S-behavior} There is a constant $c >0$ such that for every distribution
        $p$ on $\N$,
        \[
            S(p)\;<\;c.
        \]
        Moreover, for every $\epsilon > 0$, there exists a distribution $p$ on
        $\N$ such that
        \[
            S(p)\;<\;\epsilon.
        \]
    \end{lemma}
    \paragraph{Remark.}
Theorem~2.4 shows that the distribution-dependent behavior of
\(\|\widehat p-p\|_\infty\) is governed by two different features of the
ordered mass profile \(p^\downarrow\).  The term \(S(p)\) measures the
largest local variance block: at rank \(j\), the quantity
\(p^\downarrow_j\log(j+1)\) is the cost of simultaneously controlling
roughly the first \(j\) coordinates when each has mass at least
\(p^\downarrow_j\).  Thus \(\sqrt{S(p)/n}\) is the sub-Gaussian part of
the error.  It dominates when some block of coordinates carries enough probability mass
for ordinary variance fluctuations to be the main source of error.
The term \(M(p)\) captures a different phenomenon.  It is sensitive to
how many low-probability symbols are still visible at sample size \(n\).
Such symbols may each have very small variance, but because there can be
many of them, the maximum empirical fluctuation can be driven by rare
counts in the tail.  This is the sub-gamma, or Poissonian, part of the
local Glivenko--Cantelli behavior. 

The examples in Lemmas~2.5--2.6 illustrate the range of possibilities.
For a very concentrated distribution such as
\((1/2,1/2,0,\ldots)\), the tail term \(M(p)\) remains only of constant
order, and the usual \(n^{-1/2}\) variance behavior is the relevant one.
For other profiles, such as the polynomial-tail example used in the
proof of Lemma~2.5, \(M(p)\) can be as large as order \(\log n\), showing
that the tail contribution can be genuinely non-negligible.  On the
other hand, if the mass is spread over a very large effective alphabet,
as in a uniform distribution on \(A\) symbols with large \(A\), then
\(S(p)=\log(A+1)/A\) becomes very small, while \(M(p)\) stays in its
constant-scale regime.  In that case both contributions
\(\sqrt{S(p)/n}\) and \(M(p)/n\) are small.  Thus the theorem captures
three qualitatively different situations: variance domination,
tail domination, and highly spread-out profiles where both mechanisms
lead to a small error.

    \section{Related Work}
    \label{sec:related_work}

    A sharp nonasymptotic study of $\ell_{\infty}$ distribution estimation over countable
    alphabets was initiated by \citet{KonPai24}; see also the works cited therein.
    They introduced the distribution-dependent complexity proxies $v^{*}(p)$ and
    $V^{*}(p)$, and derived nearly optimal expectation and high-probability bounds
    for the MLE in terms of these quantities, along with partially empirical
    counterparts. Our refined distribution-dependent bounds also draw on the recently
    developed \emph{local Glivenko--Cantelli} framework, which seeks dimension-free,
    distribution-dependent uniform convergence rates beyond the classical worst-case
    $\sqrt{\log d/n}$ regime. \citet{CohenK23} established foundational
    characterizations and sharp rates for LGC in the product-measure setting, and
    subsequent works provided tighter regime descriptions and minimax optimality
    statements for the empirical mean estimator \citep{blanchard2023tight,BlanchardCK24,cohen25bEME}.

    \section{Proofs}

        \subsection{Proof of Theorem \ref{thm:least-favorable}}

    \begin{proof}[Proof of Theorem \ref{thm:least-favorable}]
        We prove the theorem by reducing the multinomial (dependent) sampling model
        to an independent Bernoulli model via two decoupling inequalities. For
        the independent model, sharp characterizations of the expected $\ell_{\infty}$
        error are available from \citet{blanchard2023tight}. We then maximize these
        tight bounds over all admissible distributions to identify the least-favorable
        distribution $p^{\star}$.

        Let $X^{(1)},\dots,X^{(n)}$ be independent and identically distributed random
        vectors with
        \[
            X^{(j)}\sim \mathrm{Multinomial}(1,\mathbf{p}), \qquad j=1,\dots,n,
        \]
        where $\mathbf{p}=(p_{1},p_{2},\dots)$ is a probability vector on $\mathbb{N}$.
        For each coordinate $i\ge1$, define the centered empirical deviations
        \[
            Z_{i} \coloneqq \frac{1}{n}\sum_{j=1}^{n}X^{(j)}_{i}- p_{i}, \qquad Z
            _{i}' \coloneqq p_{i} - \frac{1}{n}\sum_{j=1}^{n}X^{(j)}_{i},
        \]
        where $X^{(j)}_{i}$ denotes the $i$th coordinate of $X^{(j)}$.

        Define also $\tilde Z_{i}, \tilde Z_{i}'$ be independent random
        variables with the same marginal distributions as $Z_{i}$ and $Z_{i}'$, respectively.
        Define
        $\tilde\Delta_{n}(p) \coloneqq \E\max(\max_{i }\tilde Z_{i},\max_{i }\tilde
        Z_{i}')$.
        \begin{align*}
            \E\|\hat{\boldsymbol{p}}- \boldsymbol{p}\|_{\infty} & \quad = \E \max_{i}\left|\frac{1}{n}\sum_{j\in [n]}X_{i,j}- p_{i}\right| \\
                                                                & \quad = \E\left[\max\left(\max_{i}Z_{i}, \max_{i}Z_{i}'\right)\right].
        \end{align*}
        Combining Corollary 3 in \citet{BlanchardCK24} and
        \ref{lem:negative_covariance_multinomial},
        \begin{align*}
            \E\max_{i}Z_{i}\geq \frac{1}{2}\E\max_{i}\tilde Z_{i}, \quad \E\max_{i}Z_{i}' \geq \frac{1}{2}\E\max_{i}\tilde Z_{i}'
        \end{align*}
        Thus,
        \begin{align*}
            \E\max({\max_{i}Z_i,\max_{i}Z_i'}) & \geq \E\frac{\max_{i }Z_{i}+\max_{i}Z_{i}'}{2}                  \\
                                               & \geq \frac{\E\max_{i }\tilde Z_{i}+ \E\max_{i}\tilde Z_{i}'}{4}
        \end{align*}
        \begin{align*}
            \frac{\E\max_{i }\tilde Z_{i}+ \E\max_{i }\tilde Z_{i}'}{4} & \geq\frac{1}{4}\E\max(\max_{i }\tilde Z_{i},\max_{i }\tilde Z_{i}') \\
                                                                        & = \frac{1}{4}\tilde\Delta_{n}(p).
        \end{align*}

        By \citet[Proposition 3]{kontorovich2023decoupling},
        \[
            \frac{\mathe}{\mathe-1}\tilde\Delta_{n}(p) \geq\E\|\hat{\boldsymbol{p}}
            - \boldsymbol{p}\|_{\infty}.
        \]

        Thus we have:
        \[
            \frac{\mathe}{\mathe-1}\tilde\Delta_{n}(p) \geq\E\|\hat{\boldsymbol{p}}
            - \boldsymbol{p}\|_{\infty}\ge \frac{1}{4}\tilde\Delta_{n}(p)
        \]
        
        For the independent Bernoulli model, \citet{blanchard2023tight} proved
        that for $\mathbf{p}\in[0,1/2]^{\mathbb{N}}_{\downarrow 0}$ and
        $\hat{\boldsymbol{p}}=(\frac{1}{n}\sum_{j=1}^{n}\tilde{X}_{1j}, \dots )$
        there exist universal constants $c_{1},c_{2}>0$ such that

        \begin{align*}
            c_{1}\left( 1 \wedge\left( \sqrt{\frac{S(p)}{n}}\vee \frac{M(p)}{n}\right)\right) \leq \tilde\Delta_{n}(p)    \\
            \tilde\Delta_{n}(p) \leq c_{2} \left( 1 \wedge\left( \sqrt{\frac{S(p)}{n}}\vee \frac{M(p)}{n}\right) \right).
        \end{align*}

        Note that the restriction $\mathbf{p}\in[0,1/2]^{\mathbb{N}}_{\downarrow 0}$ is justified. Indeed, define $\mathbf{p'}$ such that for all $i$
        \[
        p'_i=
        \begin{cases}
        1-p_i & p_i > \frac{1}{2},\\
        p_i, & p_i \le \frac{1}{2}.
        \end{cases}
        \]
         Because of the symmetry of bernuli $\tilde\Delta_{n}(\mathbf{p}) = \tilde\Delta_{n}(\mathbf{p'})$. Because of $\sum_ip_i =1$ the only index of $\mathbf{p}$ that can exceed $\frac{1}{2}$ is $p_1$ and thus $\sum_ip'_i \le 1$ holds. From now on we will treat every distribution $\mathbf{p}$ as if it undergone the transformation to $\mathbf{p'}$ and thus use the assumptions $\mathbf{p}\in[0,1/2]^{\mathbb{N}}_{\downarrow 0}$, $\sum_ip_i \le 1$.

        Thus, the following holds for some constants $C_{1},C_{2}$:
        \begin{align*}
             & C_{1}
             \left( 1 \wedge\left( \sqrt{\frac{S(p)}{n}}\vee \frac{M(p)}{n}\right)\right) \leq \E\|\hat{\boldsymbol{p}}- \boldsymbol{p}\|_{\infty} \\
             & \E\|\hat{\boldsymbol{p}}- \boldsymbol{p}\|_{\infty}\leq C_{2} \left( 1 \wedge\left( \sqrt{\frac{S(p)}{n}}\vee \frac{M(p)}{n}\right) \right).
        \end{align*}

        Since $\mathbb{E}\|\hat p - p\|_{\infty}$ is bounded above and below by universal
        constant multiples of the objective in \ref{eq:max_S_M}, to find $p^{*}$
        it suffices to characterize the maximizers of this objective.
        \begin{equation}
            \label{eq:max_S_M}\sup_{j \geq 1}\Biggl( \sqrt{\frac{p_{j} \log(j+1)}{n}}
            \vee \frac{\log(j+1)}{n \log\left(2 + \frac{\log(j+1)}{np_{j}}\right)}
            \Biggr).
        \end{equation}
        
        Let $p$ be any distribution and let $j^{\star}$ attain the supremum in
        \ref{eq:max_S_M}. Because $p$ is sorted in non-increasing order and $\sum
        _{i} p_{i} \le 1$, we have $p_{j^\star}\le 1/j^{\star}$. Furthermore, for
        fixed $j$, both terms in \ref{eq:max_S_M} are monotone increasing functions
        of $p_{j}$.

        Define $U_{j^\star}$ as the uniform distribution over $\{1,\dots,j^{\star}
        \}$. Then $U_{j^\star}= 1/j^{\star} \ge p_{j^\star}$, and therefore
        \begin{align*}
            &\left( \sqrt{\frac{U_{j^\star,j^\star}\log(j^{\star}+1)}{n}}
              \vee \frac{\log(j^{\star}+1)}{n\log\!\left(2+\frac{\log(j^{\star}+1)}{n U_{j^\star,j^\star}}\right)}\right) \\
            &\quad\ge
              \left( \sqrt{\frac{p_{j^\star}\log(j^{\star}+1)}{n}}
              \vee \frac{\log(j^{\star}+1)}{n\log\!\left(2+\frac{\log(j^{\star}+1)}{n p_{j^\star}}\right)}\right).
        \end{align*}

        Hence the supremum over all distributions is attained by some uniform
        distribution on a finite support. Consequently, the search for extremal
        distributions may be restricted to the class $\{U_{j} : j \ge 2\}$. Note that $U_1$ is not relevant because $p\in[0,1/2]^{\mathbb{N}}_{\downarrow 0}$.

        It remains to determine the dimension of this uniform distribution, that
        is finding the following value
        \[
            \arg\max_{j \geq 2}\Biggl( \sqrt{\frac{\log(j+1)}{jn}}\vee \frac{\log(j+1)}{n
            \log\left(2 + \frac{j\log(j+1)}{n}\right)}\Biggr).
        \]
        Rewritten 
        \[
            \sup_{j \geq 2}\Biggl( \sqrt{\frac{\log(j+1)}{jn}}\Biggr)\vee \sup_{j
            \geq 2}\Biggl(\frac{\log(j+1)}{n \log\left(2 + \frac{j\log(j+1)}{n}\right)}
            \Biggr).
        \]
        Because $\arg\max_{j \ge2 }\frac{1}{j}\log(j+1) =2$ this becomes
        \[
            \sqrt{\frac{\log3}{2} \cdot\frac{1}{n}}\vee\frac{1}{n}\sup_{j \geq 1}
            \Biggl(\frac{\log(j+1)}{ \log\left(2 + \frac{\log(j+1)}{np_{j}}\right)}
            \Biggr).
        \]
        From \ref{lem:M-behavior} we know that for some constant $C$ and for all
        $n$, $C{\log n}\ge M(p)$. Thus, because of the faster decay in $n$ there
        exists $n_{0}$ for which for all $n \ge n_{0}$, $\sqrt{\frac{\log3}{2}
        \cdot\frac{1}{n}}>\frac{M(p)}{n}$. Thus for all $n \ge n_{0}$
        \[
            \arg\max_{j \geq 1}\Biggl( \sqrt{\frac{\log(j+1)}{jn}}\vee \frac{\log(j+1)}{n
            \log\left(2 + \frac{j\log(j+1)}{n}\right)}\Biggr) =2
        \]
        and $p^{*}= (\frac{1}{2},\frac{1}{2},0,\dots)$.
    \end{proof}

    \begin{proof}[Proof of Theorem \ref{thm:empirical-Vstar}]
        By Theorem 3 in \citet{KonPai24}, we have that with probability $\ge1 - \delta
        -\frac{81}{n}$,
        \begin{align*}
            \|p - \hat{p}\|_{\infty}
            &\le 2 \sqrt{\frac{V^{*}}{n} + \frac{v^{*}}{n} \log \frac{2}{\delta}}
            + \frac{4}{3n}\log \frac{2(n+1)}{\delta}+ \frac{\log n}{n}.
        \end{align*}
        It follows that
        \begin{align*}
            \|p - \hat{p}\|_{\infty}
            &\le 2 \sqrt{\frac{\hat{V}^{*}}{n} + \frac{\hat{v}^{*}}{n} \log \frac{2}{\delta}}
            + \sqrt{\frac{1}{n}}\sqrt{\abs{\hat{V}^{*} - V^*} + \abs{\hat{v}^{*} - v^*}
            \log \frac{2}{\delta}} \\
            &\qquad + \frac{4}{3n}\log \frac{2(n+1)}{\delta}+ \frac{\log n}{n}.
        \end{align*}
        Denote the bound in \citet[Lemma 4]{KonPai24} by
        \[
            f(n,\delta) =a + 3b^{2}/2 + b\sqrt{a}+ 3b\sqrt{\hat{v}^{*}}/2,
        \]
        where
        \begin{align*}
            a & = \frac{4}{3n}\log \frac{2(n+1)}{\delta}+ \frac{\log n}{n},      \\
            b & = 2\sqrt{\frac{\log(n+1)}{n}+ \frac{1}{n}\log \frac{2}{\delta}}.
        \end{align*}
        From Lemma \ref{lem:vstar-variation} we get
        \[
            \abs{\hat{V^*} - V^*}\le \|p-\hat{p}\|_{\infty}\log\frac{1}{\|p-\hat{p}\|_{\infty}}
            ,
        \]
        and from \ref{lem:empirical-concentration},
        \begin{align*}
            \|p - \hat{p}\|_{\infty}
            &\le 2 \sqrt{\frac{\hat{V}^{*}}{n} + \frac{\hat{v}^{*}}{n} \log \frac{2}{\delta}}
            + \sqrt{\frac{f(n,\delta)}{n}\log \frac{6n}{\delta}} \\
            &\qquad + \frac{4}{3n}\log \frac{2(n+1)}{\delta}+ \frac{\log n}{n}
        \end{align*}
        with probability $\ge1 - 2\delta - \frac{162}{n}$. Using the definitions
        of $a$ and $b$ from Lemma~\ref{lem:empirical-concentration}, we have
        \[
            f(n,\delta) \lesssim \frac{1}{n}\log\frac{n}{\delta}+ \sqrt{\frac{\hat
            v^{*}}{n}\log\frac{n}{\delta}}.
        \]
        The cross term in the expansion of $\|\hat p - p\|_{\infty}$ involves
        \begin{align*}
            T & := \sqrt{\frac{f(n,\delta)}{n}\log \frac{6n}{\delta}}                                                                     \\
              & \lesssim \frac{1}{n}\log\frac{n}{\delta}+ \left(\frac{\hat v^{*}}{n}\right)^{1/4}\frac{(\log(n/\delta))^{3/4}}{\sqrt{n}}.
        \end{align*}
        By the AM-GM inequality $\sqrt{xy}\le (x+y)/2$ with $x = \sqrt{\frac{\hat
        v^{*}}{n}\log(n/\delta)}$ and $y = \frac{1}{n}\log(n/\delta)$, the
        second term satisfies
        \begin{align*}
            \left(\frac{\hat v^{*}}{n}\right)^{1/4}\frac{(\log(n/\delta))^{3/4}}{\sqrt{n}}
            &= \sqrt{x y} \\
            &\le \frac{1}{2}\sqrt{\frac{\hat v^{*}}{n}\log\frac{n}{\delta}}
            + \frac{1}{2n}\log\frac{n}{\delta}.
        \end{align*}
        Thus, $T$ can be absorbed into the main variance term and the $O(1/n)$ term
        by adjusting the universal constants. Putting everything together,
        \begin{align*}
            \|p - \hat{p}\|_{\infty}
            &\le c_{1} \sqrt{\frac{\hat{V}^{*}}{n} + \frac{\hat{v}^{*}}{n} \log \frac{1}{\delta}}
            + c_{2} \left(\frac{\log(n/\delta)}{n}+ \frac{\log n}{n}\right).
        \end{align*}
    \end{proof}

    \subsection{Minimax result with $V^{*}$ constraint}
    \begin{proposition}[Minimax lower bound]
        \label{lem:minimax_V_star} Let $n \ge 2$ and let $V' \in (0, \frac{1}{4}\log
        2]$. Then for any estimator $\tilde{p}= \tilde{p}(X^{n})$, there exists a
        distribution $p$ supported on $\mathbb{N}$ such that $V^{*}(p) \le V'$ and
        \[
            \mathbb{E}_{p} \left[ \| \tilde{p}- p \|_{\infty} \right] \ge c \sqrt{\frac{V'}{n}}
            ,
        \]
        where $c > 0$ is a universal constant.
    \end{proposition}

    \begin{proof}
        We use Le Cam's method with two hypotheses $p^{(1)}$ and $p^{(2)}$ constructed
        to satisfy the constraint $V^{*}(p^{(j)}) \le V'$ while remaining difficult
        to distinguish.

        \paragraph{Construction of Hypotheses.}
        Let $\theta = \frac{V'}{2\log 2}$. Since $V' \le \frac{1}{4}\log 2$, we
        have $\theta \le \frac{1}{8}$. Let $\varepsilon \in (0, \theta)$. Define
        two distributions on the support $\{1, 2, 3, 4, \dots\}$:
        \begin{align*}
            p^{(1)} & = \left( \theta + \varepsilon, \theta - \varepsilon, 1 - 2\theta, 0, \dots \right), \\
            p^{(2)} & = \left( \theta - \varepsilon, \theta + \varepsilon, 1 - 2\theta, 0, \dots \right).
        \end{align*}
        We first verify the constraint $V^{*}(p^{(j)}) \le V'$. The functional
        is defined as $V^{*}(p) = \sup_{k\ge 1}p_{(k)}(1-p_{(k)})\log(k+1)$, where
        $p_{(k)}$ denotes the $k$-th largest probability. Since
        $\theta \le \frac{1}{8}$ and $\varepsilon < \theta$, the mass $\theta + \varepsilon
        < \frac{1}{4}$ is at least $1 - 2\theta > \frac{3}{4}$, making it
        strictly larger than $\theta + \varepsilon$. Thus, for both hypotheses, the
        sorted probability vector $p^{(j)\downarrow}$ is:
        \[
            p^{(j)\downarrow}= (1-2\theta, \theta+\varepsilon, \theta-\varepsilon
            , 0, \dots).
        \]
        We evaluate the term
        $v_{k} \log(k+1) = p^{(j)\downarrow}_{(k)}(1-p^{(j)\downarrow}_{(k)})\log
        (k+1)$
        for each index $k$:
        \begin{itemize}
            \item \textbf{$k=1$}: The mass is $1-2\theta$.
                \begin{align*}
                    v_{1} \log 2 & = (1-2\theta)(2\theta) \log 2 = 2\theta(1-2\theta) \log 2 \\
                                 & \le 2\theta \log 2 = V'.
                \end{align*}
                The inequality holds because $1-2\theta < 1$. By our choice of $\theta$,
                this term is exactly bounded by $V'$ (ignoring the $1-2\theta$ factor)
                and effectively sets the scale.

            \item \textbf{$k=2$}: The mass is $\theta + \varepsilon$.
                \begin{align*}
                    v_{2} \log 3 & = (\theta + \varepsilon)(1 - (\theta + \varepsilon)) \log 3 \\
                                 & \le (\theta + \varepsilon) \log 3.
                \end{align*}
                We require
                $(\theta + \varepsilon) \log 3 \le V' = 2\theta \log 2$. Since $\log
                3 \approx 1.1$ and $2\log 2 \approx 1.38$, strictly
                $\log 3 < 2\log 2$. Thus, for sufficiently small $\varepsilon$ (specifically
                $\varepsilon \le \theta( \frac{2\log 2}{\log 3}- 1)$), this term
                is strictly less than $V'$.

            \item \textbf{$k=3$}: The mass is $\theta - \varepsilon$.
                \[
                    v_{3} \log 4 = (\theta - \varepsilon)(1 - (\theta - \varepsilon
                    )) 2\log 2 < 2\theta \log 2 = V'.
                \]
        \end{itemize}
        Thus, $V^{*}(p^{(j)}) \le V'$ is satisfied for both hypotheses.

        \paragraph{Separation and KL Divergence.}
        The $\ell_{\infty}$ separation between the hypotheses is determined by the
        first two coordinates:
        \[
            \| p^{(1)}- p^{(2)}\|_{\infty} = |(\theta+\varepsilon) - (\theta-\varepsilon
            )| = 2\varepsilon.
        \]
        The KL divergence is bounded by the $\chi^{2}$ divergence:
        \begin{align*}
            D_{\mathrm{KL}}(p^{(1)}\| p^{(2)}) & \le \sum_{x \in \{1,2,3\}}\frac{(p^{(1)}_{x} - p^{(2)}_{x})^{2}}{p^{(2)}_{x}}                                                                               \\
                                               & = \frac{(2\varepsilon)^{2}}{\theta - \varepsilon}+ \frac{(-2\varepsilon)^{2}}{\theta + \varepsilon}+ 0                                                      \\
                                               & = 4\varepsilon^{2} \left( \frac{1}{\theta-\varepsilon}+ \frac{1}{\theta+\varepsilon}\right) = \frac{8\varepsilon^{2} \theta}{\theta^{2} - \varepsilon^{2}}.
        \end{align*}
        Assuming $\varepsilon \le \theta/2$, we have
        $\theta^{2} - \varepsilon^{2} \ge \frac{3}{4}\theta^{2}$, so
        \[
            D_{\mathrm{KL}}(p^{(1)}\| p^{(2)}) \le \frac{8\varepsilon^{2} \theta}{\frac{3}{4}\theta^{2}}
            = \frac{32\varepsilon^{2}}{3\theta}.
        \]
        For the tensor product of $n$ samples,
        $D_{\mathrm{KL}}((p^{(1)})^{\otimes n}\| (p^{(2)})^{\otimes n}) = n D_{\mathrm{KL}}
        (p^{(1)}\| p^{(2)})$. We set this quantity to a constant, say
        $\frac{1}{2}$:
        \[
            n \frac{32\varepsilon^{2}}{3\theta}\le \frac{1}{2}\implies \varepsilon
            ^{2} \le \frac{3\theta}{64n}\implies \varepsilon = c_{1} \sqrt{\frac{\theta}{n}}
            ,
        \]
        for some constant $c_{1} = \frac{1}{8}\sqrt{\frac{3}{2}}$. Note that for
        large $n$, $\varepsilon \ll \theta$, satisfying our earlier small
        $\varepsilon$ assumptions.

        \paragraph{Lower Bound.}
        Applying Le Cam's inequality:
        \begin{multline*}
            \inf_{\tilde{p}}\sup_{j \in \{1,2\}}\mathbb{E}_{p^{(j)}}\| \tilde{p}-
            p^{(j)}\|_{\infty} \\
            \ge \frac{\|p^{(1)}- p^{(2)}\|_{\infty}}{4}\left( 1 - \sqrt{\frac{1}{2}D_{\mathrm{KL}}((p^{(1)})^{\otimes
            n}\| (p^{(2)})^{\otimes n})}\right) \\
            \ge \frac{2\varepsilon}{4}\left( 1 - \frac{1}{2}\right) = \frac{\varepsilon}{4}
            .
        \end{multline*}
        Substituting
        $\varepsilon = c_{1} \sqrt{\frac{\theta}{n}}= c_{1} \sqrt{\frac{V'}{2n\log
        2}}$:
        \[
            \mathbb{E}_{p} \| \tilde{p}- p \|_{\infty} \ge \frac{c_{1}}{4}\sqrt{\frac{V'}{2n\log
            2}}= c \sqrt{\frac{V'}{n}},
        \]
        where $c = \frac{c_{1}}{4\sqrt{2\log 2}}$ is a universal constant. This
        confirms the lower bound rate of $\sqrt{\frac{V'}{n}}$.
    \end{proof}

    \subsection{Proof of Lemma \ref{lem:M-behavior}}
    \begin{proof}
        \textbf{Lower bound:} Assume for contradiction that $M(p) < \tfrac13$. Then
        for every $j$,
        \begin{align*}
            \frac{\log(j+1)}{\log\!\left(2+\frac{\log(j+1)}{p_{j}n}\right)}    & < \frac{1}{3} \\
            \Longrightarrow\quad \log\!\left(2+\frac{\log(j+1)}{p_{j}n}\right) & > 3\log(j+1).
        \end{align*}
        Exponentiating gives
        \begin{align*}
            2+\frac{\log(j+1)}{p_{j}n} & > (j+1)^{3}                                 \\
            \Longrightarrow\quad p_{j} & < \frac{\log(j+1)}{n\big((j+1)^{3}-2\big)}.
        \end{align*}
        Summing over $j$ yields
        \[
            1=\sum_{j\ge 1}p_{j}\;<\; \frac{1}{n}\sum_{j\ge 1}\frac{\log(j+1)}{(j+1)^{3}-2}
            .
        \]
        Let $k=j+1\ge2$. Then the sum is
        \begin{align*}
            \sum_{k\ge2}\frac{\log k}{k^{3}-2} & \;\le\; \frac{4}{3}\sum_{k\ge2}\frac{\log k}{k^{3}}                                            \\
                                               & \;\le\; \frac{4}{3}\left(\frac{\log 2}{2^{3}}+\int_{2}^{\infty}\frac{\log x}{x^{3}}\,dx\right) \\
                                               & \;=\; \frac{4}{3}\left(\frac{\log2}{8}+\frac{\log2}{8}+\frac{1}{16}\right) \;<\; 1.
        \end{align*}
        Therefore $\sum_{j}p_{j}< \frac{1}{n}\cdot 1 \le 1$, a contradiction. Hence
        $M(p)\ge \tfrac13$.

        \textbf{Upper bound:} Look at vector $p$ such that $p_{i}=\frac{1}{i}$ for
        all $i\in \mathbb{N}$. The functional $M(p)$ is defined for this vector,
        in spite of it not being a distribution.

        We will first prove that for all $n \ge2$, $M(p) < 5\log n$.

        Assume for contradiction that there exists $n'\ge2$ such that for that $n$
        \begin{equation}
            \label{eq:bound_M_negative_assumption}M(p) > 5 \log n'.
        \end{equation}

        Note that for all $n\ge1$
        \[
            \lim_{j \to \infty}\frac{\log(j+1)}{\log\!\left( 2 + j\frac{\log(j+1)}{n}
            \right)}\le \lim_{j \to \infty}\frac{\log(j+1)}{\log\!\left( 1 + j \right)}
            =1.
        \]
        Thus we can see that for all $n\ge1$ the supremum in $M(p)$ is achieved.

        For all $n\ge1$ we can denote $f(n)$ as the $j$ value that is achieved by
        the supremum for that $n$. So for a given $n$
        \[
            M(p) =\frac{\log(f(n)+1)}{\log\!\left( 2 + \frac{f(n)\log(f(n)+1)}{n}
            \right)}.
        \]
        We show that for the previously mentioned $n'$, $f(n') > (n')^{2}$. Suppose,
        for contradiction, that $f(n') < (n')^{2}$. Then
        \begin{align*}
            M(p)
            &\le \frac{\log((n')^{2}+1)}{\log\!\left( 2 + \frac{f(n')\log(f(n')+1)}{n'} \right)}
            \le \frac{2\log(2n')}{\log(2)} \\
            &= 2\frac{\log(n')}{\log(2)}+1 \le 5\log(n').
        \end{align*}
        The last inequality uses $n' \ge 2$, giving a contradiction to
        \ref{eq:bound_M_negative_assumption}.
        Thus, $f(n') > (n')^{2}$. We can also see,
        \begin{align*}
            M(p) & \le \frac{\log(f(n')+1)}{\log(\frac{f(n')}{n'})}\le \frac{\frac{3}{2}\log(f(n'))}{\log(\frac{f(n')}{n'})} \\
                 & = \frac{3}{2}\left(1+\frac{\log n'}{\log(f(n')) -\log(n')}\right)                                         \\
                 & \le \frac{3}{2}\left(1+\frac{\log n'}{2\log(n') -\log(n')}\right) = 3.
        \end{align*}
        The second transition uses $f(n') > 4$. Next, $5\log x > 3$ for
        $x \in [2, \infty)$, again contradicting
        \ref{eq:bound_M_negative_assumption}.

        To complete the proof, we show that for all $p' \in \mathcal{P}(\mathbb{N})$
        and all $n \ge 2$, $M(p) \ge M(p')$.
        Fix $p'\in \mathcal{P}{}(\mathbb{N})$ and $n\ge2$. Without the loss of
        generality we treat $p$ as a sorted distribution. $\sum_{j} p'_{j} =1$ and
        because it is a sorted distribution, for all $j \in \mathbb{N}$ we have
        $p'_{j}\le \frac{1}{j}=p_{j}$.
        In the case that the sup is achieved in $M(p')$ we denote $f(n)$ to be the
        function that gives us the $j$ that is chosen by the sup in $M(p')$.
        Thus,
        \begin{align*}
            M(p') & =\frac{\log(f(n)+1)}{\log\!\left( 2 + \frac{\log(f(n)+1)}{n p'_{f(n)}} \right)}         \\
                  & \le \frac{\log(f(n)+1)}{\log\!\left( 2 + \frac{\log(f(n)+1)}{n p_{f(n)}} \right)}       \\
                  & \le \sup_{j}\frac{\log(j+1)}{\log\!\left( 2 + \frac{\log(j+1)}{n p_{j}} \right)}= M(p).
        \end{align*}
        In the case that the sup not achieved in $M(p')$, using what we showed
        before,
        \begin{align}
            M(p') & = \lim_{j\to\infty}\frac{\log(j+1)}{\log\!\left( 2 + \frac{\log(j+1)}{p'_{j} n}\right)} \\
                  & \le \lim_{j\to\infty}\frac{\log(j+1)}{\log\!\left( 2 + j\frac{\log(j+1)}{n}\right)}     \\
                  & \le 1 \le 5\log n.
        \end{align}

        \begin{remark}
            One can sharpen the constant: the best universal lower bound
            $\inf_{n\ge1}\inf_{p}M(p)$ is $>0$ and is achieved at $n=1$; numerically
            it is about $0.48$. For $n$ large, $\inf_{p}M(p)$ increases and approaches
            $1$ from below.
        \end{remark}

        Now we show that the upper and lower bounds are achieved for some distributions.
        We look at $p_{1}$ such that $p_{i}= c\frac{1}{i^{2}}$ for all $i \ge 1$
        and $c = \frac{6}{\pi^{2}}$.

        Fix a sample size $n\ge 9$.\\
        We look at $i = \sqrt{\frac{n}{\log n}}$
        \begin{align*}
            M(p) & \ge \frac{\log(\sqrt{n/\log n} +1)}{\log\!\left( 2 + \frac{n}{\log n}\frac{\log(\sqrt{n/\log n}+1)}{c n} \right)} \\
                 & \ge \frac{\log(\sqrt{n/\log n})}{\log\!\left( 2 + \frac{\log(\sqrt{n/\log n}+1)}{c \log n} \right)}
        \end{align*}
        Furthermore because it is trivial to prove that
        $\frac{\log(\sqrt{x/\log x}+1)}{\log x}\le 1$ and
        $\frac{1}{2}\log x \le \log x - \log \log x$ for all $x\ge \mathe^{2}$:
        \begin{align*}
            \frac{\log(\sqrt{n/\log n})}{\log\!\left( 2 + \frac{\log(\sqrt{n/\log n}+1)}{c \log n} \right)}
            &\ge \frac{1}{2}\frac{\log(n) - \log\log n}{\log\!\left( 2 + \frac{1}{c} \right)} \\
            &\ge \frac{\mathe^{2}}{4\log(2+1/c)}\log(n).
        \end{align*}
        Now look at $p_{2}= (\frac{1}{2},\frac{1}{2},0 \dots)$
        \[
            M(p) = \frac{\log(3)}{\log(2 + \frac{2\log(3)}{n})}\le \frac{\log(3)}{\log(2)
            }.
        \]
    \end{proof}

    \subsection{Proof of Lemma \ref{lem:S-behavior}}
    \begin{proof}
        \textbf{Upper bound.} Recall that
        $S(p) = \sup_{j\ge 1}p_{j}^{\downarrow}\log(j+1)$. Since $p^{\downarrow}$
        is non-increasing and sums to 1, we have
        \[
            1 = \sum_{k=1}^{\infty}p_{k}^{\downarrow}\ge \sum_{k=1}^{j}p_{k}^{\downarrow}
            \ge j p_{j}^{\downarrow}\implies p_{j}^{\downarrow}\le \frac{1}{j}.
        \]
        Thus,
        \[
            p_{j}^{\downarrow}\log(j+1) \le \frac{\log(j+1)}{j}.
        \]
        The function $f(x) = \frac{\log(x+1)}{x}$ for $x \ge 1$ is maximized at $x
        \approx 2.16$ (since $f'(x) = \frac{x/(x+1) - \log(x+1)}{x^{2}}$, setting
        numerator to 0). For integer $j$, $f(1) = \log 2 \approx 0.69$, $f(2) = (
        \log 3)/2 \approx 0.55$. So $\sup_{j\ge 1}\frac{\log(j+1)}{j}\le \log 2$.
        Hence $S(p) \lesssim 1$.

        \textbf{Lower bound (unbounded near 0).} Consider the uniform
        distribution on $n$ elements, $u_{n}= (1/n, \dots, 1/n, 0, \dots)$. Then
        \[
            S(u_{n}) = \sup_{1\le j\le n}\frac{1}{n}\log(j+1) = \frac{\log(n+1)}{n}
            .
        \]
        As $n \to \infty$, $S(u_{n}) \to 0$. Thus $S(p)$ can be arbitrarily
        small.
    \end{proof}

    \subsection{Proof of Theorem \ref{thm:lgc-bound}}
    \begin{proof}
        The proof follows by refining the instance-dependent framework established
        in \citet{blanchard2023tight}, specifically their treatment of the local
        Glivenko-Cantelli (LGC) functionals. For a distribution
        $p \in \mathcal{P}(\mathbb{N})$, assume without loss of generality that $p
        = p^{\downarrow}$.

        Let $X^{n}$ be a sample of $n$ i.i.d.\ observations from $p$, and let $\hat
        p$ be the empirical distribution. By permutation invariance of the sup-norm,
        we assume without loss of generality that $p = p^{\downarrow}$.

        Define the one-sided maximum deviations $\Delta_{n}^{+}\eqdef \sup_{i
        \ge 1}(\hat p_{i}- p_{i})$ and $\Delta_{n}^{-}\eqdef \sup_{i \ge 1}(p_{i}
        - \hat p_{i})$.

        \paragraph{Coordinate-wise deviation level.}
        Following \citet{blanchard2023tight}, define the base deviation level
        for coordinate $i$ as
        \[
            \epsilon_{i}\eqdef \inf \left\{ \epsilon \ge 0 : \mathbb{P}\left(\hat
            p_{i}\ge p_{i}+ \epsilon\right) \le \frac{c_{0}}{2i}\right\},
        \]
        where $c_{0}\in (0, 1/4)$ is the universal constant from the anti-concentration
        bound in \lemref{lem:zhang-zhou}. Let
        $\epsilon \eqdef \sup_{i \ge 1}\epsilon_{i}$. By definition, for every
        index $i$, the event $\hat p_{i}- p_{i}\ge \epsilon$ occurs with probability
        at most $c_{0}/(2i)$.

        \paragraph{High-probability tail control.}
        We invoke the anti-concentration bound for $\beta =2$ (\lemref{lem:zhang-zhou})
        which states that for $\hat p_{i}\sim n^{-1}\mathrm{Bin}(n, p_{i})$,
        there exists $C' \ge 1$ such that
        \[
            \mathbb{P}(\hat p_{i}\ge p_{i}+ \epsilon) \ge c_{0}e^{-C' n \kl{p_i + \epsilon}{p_i}}
            .
        \]
        Combining this with the definition of $\epsilon_{i}$, we have for the adjusted
        term $\tilde{\epsilon}\eqdef (\epsilon \wedge 1/4) \vee (1/n)$:
        \[
            \mathbb{P}(\hat p_{i}\ge p_{i}+ \epsilon) \ge\mathbb{P}(\hat p_{i}\ge
            p_{i}+ \tilde\epsilon) \ge c_{0}e^{-C' n \kl{p_i + \tilde\epsilon}{p_i}}
            .
        \]
        \[
            \kl{p_i + \tilde{\epsilon}}{p_i}\ge \frac{\log(2i)}{n C'}.
        \]
        By the convexity of the Kullback-Leibler divergence
        $q \mapsto \kl{q}{p}$, for any scaling factor $k \ge 1$, it holds that
        $\kl{p_i + k C' \tilde{\epsilon}}{p_i}\ge k C' \kl{p_i + \tilde{\epsilon}}
        {p_i}\ge k \frac{\log(2i)}{n}$. Applying the Chernoff bound (\lemref{lem:chernoff})
        yields
        \begin{equation}
            \label{eq:tail-prob-i}\mathbb{P}(\hat p_{i}- p_{i}\ge k C' \tilde{\epsilon}
            ) \le e^{-n \kl{p_i + k C' \tilde{\epsilon}}{p_i}}\le \frac{1}{(2i)^{k}}
            .
        \end{equation}

        \paragraph{Union bound and summability.}
        Fix $\delta \in (0, e^{-2}]$ and let
        $k \supr{\delta}= \frac{\log(2/\delta)}{\log 2}$. Note that $\delta \le e
        ^{-2}$ implies $k \supr{\delta}\ge 2$. Applying the union bound over all
        coordinates $i \ge 1$:
        \begin{align*}
            \mathbb{P}(\Delta_{n}^{+}\ge k \supr{\delta}C' \tilde{\epsilon}) & \le \sum_{i=1}^{\infty}\mathbb{P}(\hat p_{i}- p_{i}\ge k \supr{\delta}C' \tilde{\epsilon})                                                                            \\
                                                                             & \maavar{\ref{eq:tail-prob-i}}{\le}\sum_{i=1}^{\infty}\frac{1}{(2i)^{k \supr{\delta}}}= \frac{1}{2^{k \supr{\delta}}}\sum_{i=1}^{\infty}\frac{1}{i^{k \supr{\delta}}}.
        \end{align*}
        For $k \supr{\delta}\ge 2$, the sum is bounded by
        $1 + \int_{1}^{\infty}x^{-2}dx = 2$. Thus,
        \[
            \mathbb{P}(\Delta_{n}^{+}\ge k \supr{\delta}C' \tilde{\epsilon}) \le
            \frac{2}{2^{k \supr{\delta}}}= \frac{2}{2^{\log(2/\delta)/\log 2}}= \delta
            .
        \]
        This establishes that $\Delta_{n}^{+}\le C \epsilon \log(1/\delta)$ with
        probability at least $1-\delta$.

            We first note a simple symmetry that allows us to restrict attention to the case
            $p_i \le 1/2$. For each coordinate,
            \[
            |\hat p_i-p_i| \stackrel{d}{=}
            \left|n^{-1}\mathrm{Bin}(n,p_i)-p_i\right|,
            \]
            and this distribution is unchanged when $p_i$ is replaced by $1-p_i$. Hence the coordinatewise tail probabilities appearing
            in the union bound are unchanged if, in that step, we replace $p$ by the vector
            $p'$ defined by
            \[
            p'_i=
            \begin{cases}
            1-p_i, & p_i>1/2,\\
            p_i, & p_i\le 1/2.
            \end{cases}
            \]
            Moreover, $p'_i\le p_i$ for every $i$, and therefore also
            $(p')^\downarrow_i\le p^\downarrow_i$ for all $i$. Consequently,
            \[
            S(p')\le S(p), \qquad M(p')\le M(p).
            \]
            Thus, proving the desired union-bound estimate under the assumption
            $p_i\in[0,1/2]$ for all $i$ also proves it for the original vector $p$.
            
        \paragraph{Symmetry and left-tail behavior.}
        Since $p_{i}\le 1/2$, the binomial distribution $\mathrm{Bin}(n, p_{i})$
        has a heavier right tail than left tail. Formally, for any
        $\epsilon > 0$, $\kl{p_i - \epsilon}{p_i}\ge \kl{p_i + \epsilon}{p_i}$ \citep{garivier2011kl}.
        Consequently, the same union bound argument applies to $\Delta_{n}^{-}$,
        yielding
        $\mathbb{P}(\Delta_{n}^{-}\ge C \epsilon \log(1/\delta)) \le \delta$. By
        the union bound over the left and right events,
        $\|\hat p - p\|_{\infty}\le C \epsilon \log(1/\delta)$ with probability
        at least $1-2\delta$.

        \paragraph{Characterization of the deviation level.}
        It remains to relate $\epsilon = \sup_{i}\epsilon_{i}$ to the functionals
        $S(p)$ and $M(p)$. By Proposition 13 in \citet{blanchard2023tight}, the
        deviation level satisfies $\epsilon_{i}\le C_{1}\phi_{i,p_i}(n)$. From the
        definition of $\phi$, the supremum over $i$ is bounded by the sub-Gaussian
        and sub-gamma terms, yielding
        \[
            \epsilon \;\lesssim\; \left(\sqrt{\frac{S(p)}{n}}\vee \frac{M(p)}{n}\right
            ) \wedge 1.
        \]
        Recall that $\tilde{\epsilon}= (\epsilon \wedge 1/4) \vee (1/n)$. By
        \lemref{lem:M-behavior}, $M(p) \gtrsim 1$, which implies $\epsilon \gtrsim
        1/n$. Since we also have $\epsilon \lesssim 1$ in the range of interest,
        it follows that $\tilde{\epsilon}\approx \epsilon$, establishing the
        final bound.
    \end{proof}

    \subsubsection{Proof of Lemma \ref{lem:M-dominates-S}}
    \begin{lemma}[M-term can dominate S-term]
        \label{lem:M-dominates-S} There exist distributions $p$ for which $M(p)$
        strictly dominates $\sqrt{S(p)}$ for large $n$.
    \end{lemma}

    \begin{proof}[Proof of Lemma \ref{lem:M-dominates-S}]
        Fix a sample size $n\ge1$. Consider the uniform distribution $p$ such
        that $p_{j} = \frac{1}{n}$ for $j \le n$.
        \[
            \frac{M(p)}{n}
            = \frac{\log(n+1)}{n}\left[\log\bigl(2 + \log(n+1)\bigr)\right]^{-1}.
        \]
        The index $n$ is chosen because $\frac{\log x}{\log(2+ \log x)}$ is an increasing
        function. Meanwhile, the $S$-term is
        \[
            \sqrt{\frac{S(p)}{n}}= \sqrt{\frac{\log n}{n^{2}}}.
        \]
        Since $M(p)/n$ is larger than $\sqrt{S(p)/n}$ for large $n$, with
        $M(p) > \sqrt{S(p)}$ (in term comparison) at $n = 1$, the result follows.
    \end{proof}








    \section*{Impact Statement}

    This paper presents work whose goal is to advance the field of Machine Learning.
    There are many potential societal consequences of our work, none which we feel
    must be specifically highlighted here.

    \bibliography{master,additional_refs}
    \bibliographystyle{plainnat}

    \newpage
    \appendix

    \section{Additional Results \& Auxiliary Lemmas}
    \begin{lemma}[Negative covariance of multinomial components; \citet{Dubhashi:1998:BBS:299633.299634}]
        \label{lem:negative_covariance_multinomial} Let $\{X^{(j)}\}_{j=1}^{n}$ be
        $n$ i.i.d.\ samples from a multinomial distribution with parameter
        vector $\mathbf{p}=(p_{1},p_{2},\dots)$ and single trial (i.e., $X^{(j)}\sim
        \mathrm{Multinomial}(1, \mathbf{p})$). Let $X_{i,j}$ denote the $i$-th component
        of the vector $X^{(j)}$. Define
        \[
            Z_{i}= \frac{1}{n}\sum_{j=1}^{n}X_{i,j}- p_{i}, \qquad Z'_{i}= p_{i}-
            \frac{1}{n}\sum_{j=1}^{n}X_{i,j}.
        \]
        Then for all distinct $i \neq k$ with $p_{i}, p_{k} > 0$,
        \begin{equation}
            \label{eq:neg_cov}\operatorname{Cov}(Z_{i},Z_{k}) = -\frac{p_{i} p_{k}}{n}
            < 0 \quad\text{and}\quad \operatorname{Cov}(Z'_{i},Z'_{k}) < 0.
        \end{equation}
        Moreover, the variables $\{Z_{i}\}_{i}$ are Negatively Associated in the
        sense of \citet{Dubhashi:1998:BBS:299633.299634}.
    \end{lemma}

        \subsection{Proof of Theorem \ref{thm:empirical-Vstar}}
    This lemma establishes a continuity-like property for V*, which is crucial for
    relating the population functional to its empirical version in the high-probability
    bound.
    \subsubsection{Proof of Lemma \ref{lem:vstar-variation}}
    \begin{lemma}[Distribution-dependent bound on $V^{*}$]
        \label{lem:vstar-variation} For two distributions $p,q$, we have
        \[
            \Delta(p,q)\le \bar\delta(p,q)\log\frac{1}{\bar\delta(p,q)},
        \]
        where $\Delta(p,q):=\abs{V^*(p)-V^*(q)}$ and
        $\bar\delta(p,q):=\sup_{i\in\N}\abs{v_i(p^\downarrow)-v_i(q^\downarrow)}$.
        In particular, since $\bar\delta(p,q) \le \|p-q\|_{\infty}$, we have $\Delta
        (p,q)\le \|p-q\|_{\infty} \log(1/\|p-q\|_{\infty})$.
    \end{lemma}
    \begin{proof}
        Recall that $V^{*}(p) = \sup_{i \ge 1}v_{i}(p^{\downarrow})\log(i+1)$.
        Using the inequality
        $|\sup_{i} a_{i} - \sup_{i} b_{i}| \le \sup_{i} |a_{i} - b_{i}|$, we have
        \[
            \Delta(p,q) \le \sup_{i \ge 1}\abs{v_i(p^\downarrow) \log(i+1) - v_i(q^\downarrow) \log(i+1)}
            .
        \]
        Let $\delta_{i} = \abs{v_i(p^\downarrow) - v_i(q^\downarrow)}$. From
        \citet[Lemma 12]{KonPai24}, we know that $v_{i}(p^{\downarrow}) \le \frac{1}{i+1}$
        for any distribution $p$. Since $x,y \ge 0 \implies |x-y| \le \max(x,y)$,
        we have
        \[
            \delta_{i} \le \max(v_{i}(p^{\downarrow}), v_{i}(q^{\downarrow})) \le
            \frac{1}{i+1},
        \]
        which implies $\log(i+1) \le \log(1/\delta_{i})$. The function
        $f(x) = x\log(1/x)$ is increasing on $(0, 1/e]$. Since $\delta_{i} \le 1/
        2$, we have
        \[
            \delta_{i} \log(i+1) \le \delta_{i} \log \frac{1}{\delta_{i}}\le \bar
            \delta(p,q) \log \frac{1}{\bar\delta(p,q)}.
        \]
        Taking the supremum over $i$ yields the first claim. For the second,
        note that
        $\abs{v_i(p^\downarrow) - v_i(q^\downarrow)}\le \abs{p^\downarrow_i - q^\downarrow_i}
        \le \|p^{\downarrow} - q^{\downarrow}\|_{\infty} \le \|p-q\|_{\infty}$.
    \end{proof}

        \subsubsection{Proof of Lemma \ref{lem:empirical-concentration}}
    \begin{lemma}[Empirical functional concentration]
        \label{lem:empirical-concentration} With probability $\ge1 - \delta - \frac{81}{n}$,
        we have
        \[
            \delta(p,\hat{p}) \leq a + \frac{3b^{2}}{2}+ b\sqrt{a}+ \frac{3b\sqrt{\hat{v}^*}}{2}
            ,
        \]
        where
        \begin{align*}
            a & = \frac{4}{3n}\log \frac{2(n+1)}{\delta}+ \frac{\log n}{n},      \\
            b & = 2\sqrt{\frac{\log(n+1)}{n}+ \frac{1}{n}\log \frac{2}{\delta}}.
        \end{align*}
    \end{lemma}
    \begin{proof}
        We will show that
        \[
            \delta(p,\hat{p}) \le \lVert p - \hat{p}\rVert_{\infty}
            \le a + b\sqrt{\hat{v}^{*}} + b\sqrt{\delta(p,\hat{p})}.
        \]
        and from there apply the proof technique from \citet[Theorem 4]{KonPai24}.
        First, it follows from that proof that
        \begin{align*}
            \lVert p - \hat{p}\rVert_{\infty} & \leq a + b\sqrt{\hat{v}^{*}}+ b\sqrt{\lvert v^{*}- \hat{v}^{*}\rvert} \\
                                              & \leq a + b\sqrt{\hat{v}^{*}}+ b\sqrt{\delta(p,\hat{p})}.
        \end{align*}
        Next, for all distributions $p,q$ we have
        \begin{align*}
            \abs{p_i(1- p_i) -q_i(1-q_i)}
            &=\abs{p_i -q_i +(q_i - p_i)(q_i + p_i)} \\
            &\le \abs{p_i -q_i}.
        \end{align*}
        Going back to $\delta(p,\hat{p})$, where $i$ is the index chosen by the supremum
        (note that $\sup$ and $\max$ are equivalent in this case)
        \begin{align*}
            \delta(p,\hat{p}) & = \sup_{i\in \mathbb{N}}\abs{p_i(1- p_i) -q_i(1-q_i)}      \\
                              & \leq \abs{p_i -q_i}\leq \lVert p - \hat{p}\rVert_{\infty}.
        \end{align*}
        Now we have:
        \[
            A \le B\sqrt{A}+ C,
        \]
        where $A = \delta(p,\hat{p})$, $B = b$, $C = a + b\sqrt{\hat v^{*}}$, which
        implies $A \le B^{2} + B\sqrt{C}+ C$, or
        \[
            \delta(p,\hat{p}) \le b^{2} + a + b\sqrt{\hat v^{*}}+ b\sqrt{a + b\sqrt{\hat
            v^{*}}}.
        \]

        Using $\sqrt{x+y}\le \sqrt{x}+ \sqrt{y}$ and $\sqrt{xy}\le (x+y)/2$,

        \begin{align*}
            \delta(p,\hat{p})
            &\le b^{2} + a + b\sqrt{\hat v^{*}}+ b\sqrt{a}+ b\sqrt{b\sqrt{\hat v^{*}}} \\
            &\le b^{2} + a + b\sqrt{\hat v^{*}}+ b\sqrt{a}+ \frac{b(b + \sqrt{\hat v^{*}})}{2} \\
            &= a + \frac{3b^{2}}{2}+ b\sqrt{a}+ \frac{3b\sqrt{\hat v^*}}{2}.
        \end{align*}

        We still have
        \[
            a + b\sqrt{v^{*}}\le a + \frac{3b^{2}}{2}+ b\sqrt{a}+ \frac{3b\sqrt{\hat
            v^*}}{2}.
        \]
        whence, with probability $1-\delta - \frac{81}{n}$,
        \[
            \|p - \hat p\|_{\infty}\le a + \frac{3b^{2}}{2}+ b\sqrt{a}+ \frac{3b\sqrt{\hat
            v^*}}{2}.
        \]
    \end{proof}

        \begin{lemma}[Global Scaling of Bernoulli KL Divergence]
        \label{lem:kl-scaling-global} Fix $p \in (0, 1/2]$. For any $x > 0$ and
        $y \ge 1$ satisfying $p+xy \le 1/2$, the following inequality holds:
        \[
            \kl{p+xy}{p}\le 8 y^{2} \kl{p+x}{p}.
        \]
    \end{lemma}

    \begin{proof}
        Let $u := x/p$. We seek to bound the ratio of divergences. We first establish
        global bounds for the numerator and denominator.

        \paragraph{Step 1: Global lower bound for the denominator.}
        Let $t > 0$ such that $p+t \le 1/2$. We claim that
        \begin{equation}
            \label{eq:global_lb}\kl{p+t}{p}\ge \frac{t}{2}\log\left(1+\frac{t}{p}
            \right).
        \end{equation}
        \textit{Proof of \eqref{eq:global_lb}:} Using the inequality
        $\log(1-v) \ge -v/(1-v)$ for $v < 1$, we have
        \begin{align*}
            \kl{p+t}{p} & = (p+t)\log(1+t/p)                              \\
                        & \quad + (1-p-t)\log\left(1-\frac{t}{1-p}\right) \\
                        & \ge (p+t)\log(1+t/p) - t.
        \end{align*}
        Let $z = t/p$. The RHS is $p [ (1+z)\log(1+z) - z ]$. Consider the function
        \begin{align*}
            g(z) & = (1+z)\log(1+z) - z - \frac{z}{2}\log(1+z) \\
                 & = \left(1+\frac{z}{2}\right)\log(1+z) - z.
        \end{align*}
        Differentiation yields
        \begin{align*}
            g'(z) & = \frac{1}{2}\log(1+z) + \frac{1+z/2}{1+z}- 1 \\
                  & = \frac{1}{2}\log(1+z) - \frac{z}{2(1+z)}.
        \end{align*}
        Since $\log(1+z) \ge \frac{z}{1+z}$, we have $g'(z) \ge 0$. Since $g(0)=0$,
        $g(z) \ge 0$ for all $z \ge 0$. Thus,
        $\kl{p+t}{p}\ge p \cdot \frac{z}{2}\log(1+z) = \frac{t}{2}\log(1+t/p)$.

        \paragraph{Step 2: Case analysis.}
        We split the analysis based on the magnitude of the shift $xy$ relative to
        $p$.

        \textbf{Case 1: $xy \le p$.} This implies $x \le p$ (since $y \ge 1$). For
        the numerator, we use the quadratic upper bound derived from Taylor
        expansion. Since $p \le 1/2$, we have $1-p \ge 1/2$. Using $\log(1+a) \le
        a$ and $\log(1-b) \le -b$:
        \begin{align*}
            \kl{p+xy}{p} & \le \frac{(xy)^{2}}{p}+ \frac{(xy)^{2}}{1-p}                        \\
                         & \le \frac{(xy)^{2}}{p}+ \frac{2(xy)^{2}}{p}\le \frac{2(xy)^{2}}{p}.
        \end{align*}
        For the denominator, since $x \le p$, we have $u = x/p \le 1$. Using $\log
        (1+u) \ge u/2$ for $u \in [0,1]$ in \eqref{eq:global_lb}:
        \[
            \kl{p+x}{p}\ge \frac{x}{2}\cdot \frac{u}{2}= \frac{x^{2}}{4p}.
        \]
        The ratio is bounded by:
        \[
            \frac{\kl{p+xy}{p}}{\kl{p+x}{p}}\le \frac{2 x^{2} y^{2} / p}{x^{2} /
            4p}= 8y^{2}.
        \]

        \textbf{Case 2: $xy > p$.} In this regime, we use the bound $\kl{p+t}{p}\le
        (p+t)\log(1+t/p)$. Thus,
        \[
            \kl{p+xy}{p}\le (p+xy)\log(1+xy/p).
        \]
        We distinguish two sub-cases based on $x$.

        \textit{Sub-case 2a: $x > p$.} Here, both numerator and denominator are in
        the ``linear-log'' regime. From \eqref{eq:global_lb},
        $\kl{p+x}{p}\ge \frac{x}{2}\log(1+x/p)$. Let $u = x/p > 1$.
        \[
            \frac{\kl{p+xy}{p}}{\kl{p+x}{p}}\le \frac{p(1+uy)\log(1+uy)}{\frac{p
            u}{2}\log(1+u)}.
        \]
        Since $u > 1$ and $y \ge 1$, we have $\frac{1+uy}{u}\le 2y$ and $\frac{\log(1+uy)}{\log(1+u)}
        \le 1 + \frac{\log y}{\log(1+u)}\le 1 + 1.5\log y$. Thus the ratio is at
        most $2(2y)(1+1.5\log y) = 4y + 6y\log y$. Since $y \ge 1$, we verify
        that $4y + 6y\log y \le 8y^{2}$. (At $y=1$, $4 \le 8$; for $y>1$, $y^{2}$
        grows faster than $y\log y$.)

        \textit{Sub-case 2b: $x \le p$.} Here $x$ is small but $xy$ is large
        (transition regime). Since $p<xy$, the numerator satisfies
        $\kl{p+xy}{p}\le (p+xy)\log(1+xy/p) \le 2xy\log(1+xy/p)$.
        The denominator satisfies $\kl{p+x}{p}\ge x^{2}/(4p)$, as derived in
        Case 1.
        \[
            \text{Ratio}\le \frac{2xy \log(1+xy/p)}{x^{2} / 4p}= \frac{8py}{x}\log
            \left(1+\frac{xy}{p}\right).
        \]
        Let $u = x/p \le 1$. Then Ratio $\le \frac{8y}{u}\log(1+yu)$. Since
        $\log(1+z) \le z$ for all $z$,
        \[
            \text{Ratio}\le \frac{8y}{u}(yu) = 8y^{2}.
        \]

        \paragraph{Conclusion.}
        In all cases, the ratio is bounded by $8y^{2}$.
    \end{proof}

    \begin{lemma}[Chernoff bound for Binomial variables]
        \label{lem:chernoff} Let $n\in\N$ and let $Y\sim\mathrm{Bin}(n,p)$ with $p
        \in[0,1]$. Then for every $q\in[p,1]$,
        \[
            \Pr\!\left(\frac{Y}{n}\ge q\right)\le \exp\!\bigl(-n\,\kl{q}{p}\bigr)
            ,
        \]
        where for $p,q\in[0,1]$,
        \[
            \kl{q}{p}:= q\log\!\frac{q}{p}+ (1-q)\log\!\frac{1-q}{1-p}.
        \]
    \end{lemma}

    \begin{lemma}[Anti-concentration for Binomial upper tails \cite{zz2020}]
        \label{lem:zhang-zhou} For any $\beta > 1$ there exist constants $c_{\beta}
        , C_{\beta} > 0$, depending only on $\beta$, such that

        \begin{align*}
            \mathbb{P}(\hat{p}_{n} - p \geq \epsilon) & \geq c_{\beta} \exp \paren{-C_\beta n \kl{\epsilon + p}{p}},                                    \\
                                                      & \qquad\text{if $0 \leq \epsilon \leq \tfrac{1-p}{\beta}$ and $\epsilon + p \geq \tfrac{1}{n}$}; \\
                                                      & = 1 - (1-p)^{n}, \quad\text{if $0 < \epsilon + p < \tfrac{1}{n}$}.
        \end{align*}
    \end{lemma}

    \begin{lemma}[\citet{blanchard2023tight}]
        \label{lem:kl-upper-bound} Let $0 \le q, \epsilon \le \frac{1}{4}$ and suppose
        $\epsilon \ge 8q$. Define $h(u) \eqdef (1+u)\ln(1+u)-u$. Then,
        \[
            \frac{\epsilon}{2}\ln \frac{\epsilon}{q}\le q h\left(\frac{\epsilon}{q}
            \right) \le \kl{q+\epsilon}{q}\le 2\epsilon \ln \frac{\epsilon}{q}.
        \]
        Also, for any $0 \le q, \epsilon \le 1$ with $q + \epsilon \le \frac{1}{2}$,
        \[
            \frac{\epsilon^{2}}{2(q+\epsilon)}\le q h\left(\frac{\epsilon}{q}\right
            ) \le \kl{q+\epsilon}{q}\le \frac{\epsilon^{2}}{q}.
        \]
    \end{lemma}

        \subsection{Anti-concentration Lower Bound}
    \begin{proposition}
        Let $S(p) \coloneqq \sup_{j \ge 1}p_{j} \log(j+1)$ and let $j'$ be an index
        such that $p_{j'}\log(j'+1) \ge S(p)/2$. Then for any $k > 0$, we have
        \[
            \mathbb{P}\!\left( \sup_{j}(\hat{p}_{j}- p_{j}) \ge \tau \right) \ge
            (1-e^{-1}) \left(1 \wedge \frac{c_{2}}{2(j'+1)^{2k^2-1}}\right),
        \]
        where $\tau \coloneqq \left(k\sqrt{\frac{S(p)}{2C n}}\wedge \frac{1}{4}\right
        ) \vee \frac{1}{n}$, and $c_{2}, C$ are universal constants from \lemref{lem:zhang-zhou}.
    \end{proposition}
    \begin{proof}
        Let $\tau \coloneqq \left(k\sqrt{\frac{S(p)}{2C n}}\wedge \frac{1}{4}\right
        ) \vee \frac{1}{n}$. Note that for all $j$, $p_{j}+ \tau \ge \frac{1}{n}$.
        Invoking \lemref{lem:kl-upper-bound} and \lemref{lem:zhang-zhou}, we have
        that
        \[
            \mathbb{P}(\hat{p}_{j}- p_{j} \ge \tau) \ge c_{2}\exp\!\left( - C_{2}
            n \, \kl{p_{j} + \tau}{p_{j}}\right).
        \]
        In the regime where $\tau = k\sqrt{\frac{S(p)}{2C n}}$, using \lemref{lem:kl-scaling-global}
        yields
        \begin{align*}
            \mathbb{P}(\hat{p}_{j}- p_{j} \ge \tau) & \ge c_{2}\exp\!\left( - C_{2}n \, \frac{k^{2} S(p)}{2 C n p_{j}}\right) \\
                                                    & = c_{2}\exp\!\left( - k^{2}\frac{S(p)}{p_{j}}\right).
        \end{align*}
        As discussed in \thmref{thm:least-favorable}, the estimators $\hat p_{j}$
        are negatively associated. This implies that for the increasing events
        $A_{j}=\{ \hat p_{j} - p_{j} \ge \tau \}$, we have $\mathbb{P}(\cap A_{j}
        ^{c}) \le \prod \mathbb{P}(A_{j}^{c})$. Consequently,
        \[
            \mathbb{P}(\cup A_{k}) \ge (1-e^{-1})(1 \wedge \sum \mathbb{P}(A_{k})
            ):
        \]
        \[
            \mathbb{P}\!\left( \sup_{j}(\hat{p}_{j}- p_{j}) \ge \tau\right)
            \ge (1-e^{-1})\left(1 \wedge \sum_{j} c_{2}\exp\!\left(-k^{2}\frac{S(p)}{p_{j}}\right)\right).
        \]
        By definition of $S(p) = \sup_{j \ge 1}p_{j} \log(j+1)$, there exists an
        index $j'$ such that
        \[
            p_{j'}\log(j'+1) \ge \frac{1}{2}S(p) \implies \frac{S(p)}{p_{j'}}\le
            2\log(j'+1).
        \]
        Since $p$ is sorted non-increasingly, for all $j \le j'$,
        $p_{j} \ge p_{j'}$. Thus, we can lower bound the sum by truncating at
        $j'$:
        \begin{align*}
            \sum_{j}c_{2}\exp\!\left( - k^{2}\frac{S(p)}{p_{j}}\right) & \ge \sum_{j=1}^{j'}c_{2}\exp\!\left( - k^{2}\frac{S(p)}{p_{j'}}\right) \\
                                                                       & \ge j' c_{2}\exp\!\left( - 2k^{2}\log(j'+1)\right)                     \\
                                                                       & = \frac{c_{2}j'}{(j'+1)^{2k^2}}                                        \\
                                                                       & \ge \frac{c_{2}}{2(j'+1)^{2k^2-1}}.
        \end{align*}
        Substituting this back into our lower bound for the probability of the union:
        \[
            \mathbb{P}\!\left( \sup_{j}(\hat{p}_{j}- p_{j}) \ge \tau\right)
            \ge (1-e^{-1})\left(1 \wedge \frac{c_{2}}{2(j'+1)^{2k^2-1}}\right).
        \]
        This provides a polynomial lower bound in terms of the truncation index
        $j'$, which is related to the tail decay of the distribution.
    \end{proof}

    \section{Experiments}
    \label{sec:experiments}

    We compare our fully empirical high-probability bound from \thmref{thm:empirical-Vstar}
    to the bounds of \citet{KonPai24}, reproducing their experimental protocol
    and figures, and then adding our bound as an additional curve.

    \paragraph{Protocol (matching \citet{KonPai24}).}
    For each sample size $n$, we draw $R=1000$ i.i.d.\ samples $X^{n}\sim p$, compute
    the empirical pmf $\hat p$, and evaluate each candidate upper bound on
    $\|\hat p - p\|_{\infty}$. We also plot an \emph{oracle reference curve}
    defined as the empirical $(1-\delta)$-quantile of $\|\hat p - p\|_{\infty}$
    across the $R$ repetitions (this is not a bound computable from data, but
    serves as a scale reference). All curves are shown on a log--log scale in $n$.

    \paragraph{Distributions and datasets.}
    We use the same synthetic distributions as \citet{KonPai24}: (i) the uniform
    distribution over $A=100$ symbols, and (ii) a truncated Zipf law with parameter
    $s=1.1$ over $A=100$ symbols. We also reproduce the two real-world datasets
    from \citet{KonPai24} (Surnames and Hamilton), treating the dataset's
    empirical histogram as the ground-truth distribution $p$ and repeatedly
    resampling $X^{n}$ from this $p$.

    \paragraph{Confidence levels and sample-size grids.}
    We consider two choices of confidence: (a) fixed $\delta=0.05$ (Figure~\ref{fig:fig1_ours}
    and Figure~\ref{fig:fig3_ours}), and (b) $\delta = 1/n^{2}$ (Figure~\ref{fig:fig2_ours}).
    The grids of $n$ match the reproduction notebook: Figure~\ref{fig:fig1_ours}
    uses $n \in \{10^{4},2\cdot 10^{4},\dots,2\cdot 10^{5}\}$, Figure~\ref{fig:fig2_ours}
    uses $n \in \{10^{5},2\cdot 10^{5},\dots,10^{6}\}$, and Figure~\ref{fig:fig3_ours}
    uses $n \in \{10^{4},2\cdot 10^{4},\dots,5\cdot 10^{5}\}$.

    \paragraph{Bounds compared.}
    We plot the two bounds from \citet{KonPai24} that appear in their figures (labeled
    "K--P (Thm 2)" and "K--P (Thm 4)"), their benchmark curve, and our new bound.
    Our curve ("Ours (Thm 2.2)") is obtained by evaluating the fully empirical
    plug-in quantities $\hat V^{*}$ and $\hat v^{*}$ from \secref{sec:main_results},
    as in \thmref{thm:empirical-Vstar}. For the plots we instantiate \thmref{thm:empirical-Vstar}
    using the explicit constants currently implemented in the code (a direct
    transcription of the current proof-level constants):
    \begin{align*}
        \mathrm{Bound}_{\mathrm{ours}}(\hat p,n,\delta)
        &= 2\sqrt{\frac{\hat V^{*}}{n} + \frac{\hat v^{*}}{n}\log\frac{2}{\delta}}
        + \frac{1}{n}\log\!\Bigl(\frac{6n(n+1)}{\delta}\Bigr) \\
        &\qquad + \frac{4}{3n}\log\!\Bigl(\frac{2(n+1)}{\delta}\Bigr) + \frac{\log n}{n}.
    \end{align*}
    (These constants are not optimized; the purpose of the experiments is a like-for-like
    comparison under the same plotting protocol.)

    \paragraph{Results.}
    Our fully empirical bound consistently improves upon the benchmark curve
    across all distributions and confidence levels. In settings with heavier tails
    (Zipf and real-world datasets), our bound substantially outperforms the
    competing empirical bound "K--P (Thm 4)" and approaches the tightness of the
    population-dependent "K--P (Thm 2)" bound (Figure~\ref{fig:fig1_ours} (right)
    and Figure~\ref{fig:fig3_ours}). For the uniform distribution, where the
    tail complexity is minimal, "K--P (Thm 4)" remains tighter than our bound (Figure~\ref{fig:fig1_ours}
    (left)). In the vanishing-confidence regime $\delta=1/n^{2}$, our bound is
    robust, particularly for the Zipf distribution where it nearly matches the population-based
    "K--P (Thm 2)" curve (Figure~\ref{fig:fig2_ours}).

    \begin{figure*}[t]
        \centering
        \begin{subfigure}
            [t]{\textwidth}
            \centering
            \includegraphics[width=\linewidth]{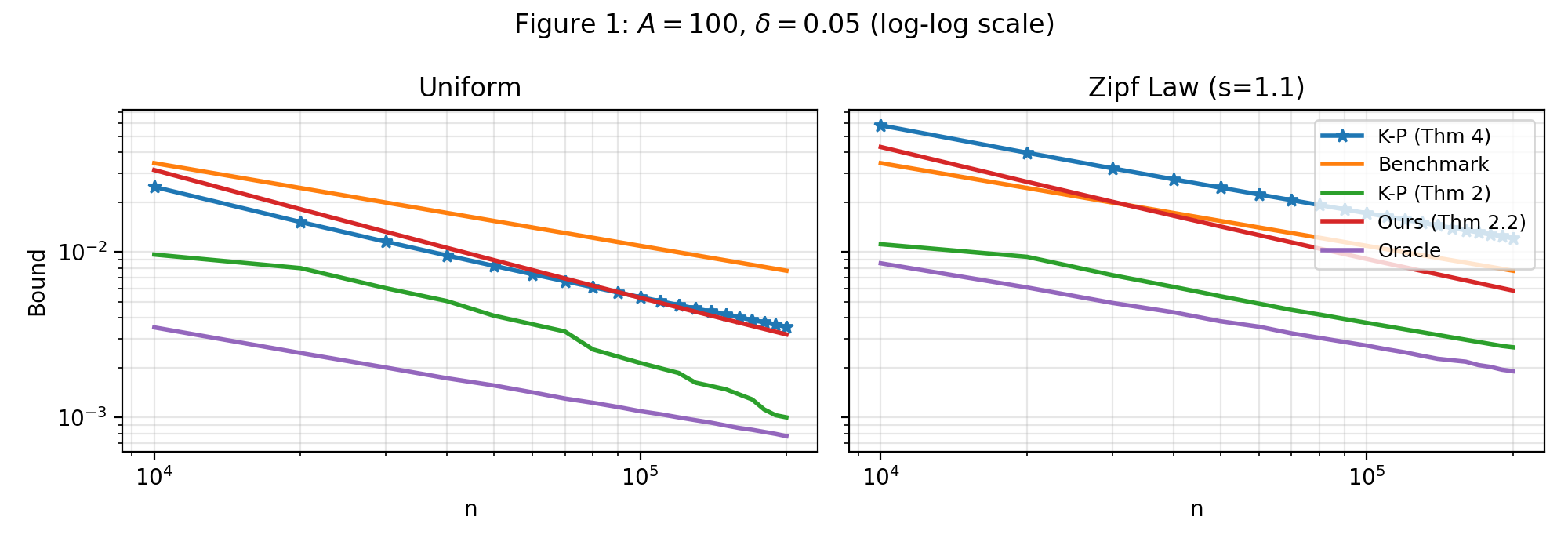}
            \caption{Synthetic distributions with $A=100$ and fixed confidence
            $\delta=0.05$ (log--log scale).}
            \label{fig:fig1_ours}
        \end{subfigure}
        \hfill
        \begin{subfigure}
            [t]{\textwidth}
            \centering
            \includegraphics[width=\linewidth]{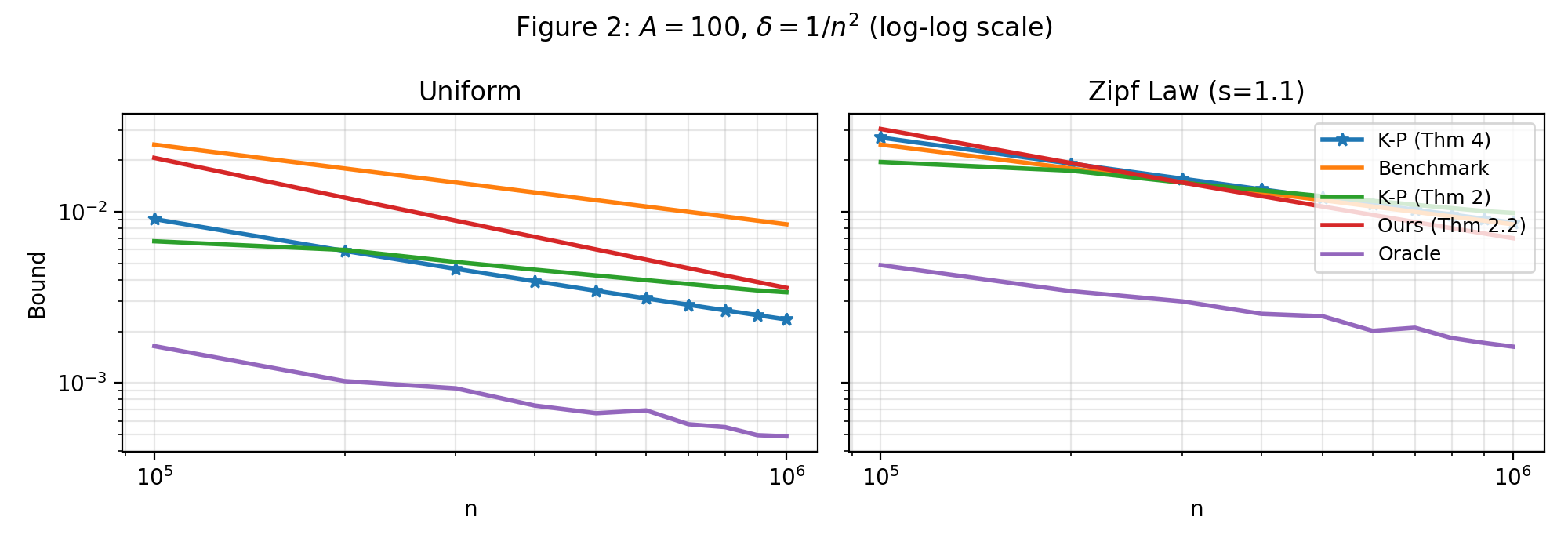}
            \caption{Synthetic distributions with $A=100$ and $\delta=1/n^{2}$ (log--log
            scale).}
            \label{fig:fig2_ours}
        \end{subfigure}

        \vspace{0.5em}

        \begin{subfigure}
            [t]{\textwidth}
            \centering
            \includegraphics[width=\linewidth]{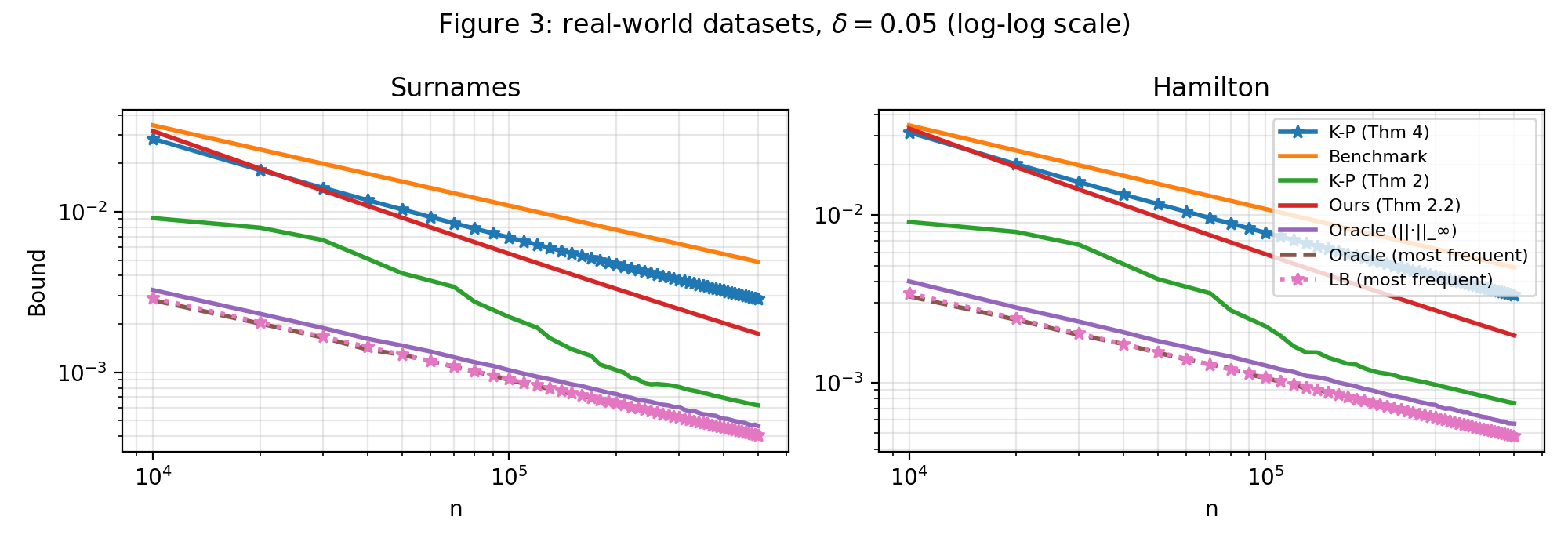}
            \caption{Real-world datasets (Surnames and Hamilton) with
            $\delta=0.05$ (log--log scale).}
            \label{fig:fig3_ours}
        \end{subfigure}
        \caption{Comparison of bounds reproducing \citet{KonPai24} and adding our
        bound from \thmref{thm:empirical-Vstar}.}
        \label{fig:all_experiments}
    \end{figure*}

\end{document}